\documentclass{article}

\usepackage{arxiv}
\usepackage{mypackage}

\usepackage[utf8]{inputenc} 
\usepackage[T1]{fontenc}    
\usepackage{hyperref}       
\usepackage{url}            
\usepackage{booktabs}       
\usepackage{amsfonts}       
\usepackage{nicefrac}       
\usepackage{microtype}      
\usepackage{lipsum}

\title{SVGD: A Virtual Gradients Descent Method for Stochastic Optimization}

\author[a]{Zheng Li}
\author[,a,b]{Shi Shu\thanks{
		Corresponding author \newline
		This work is supported by the National Natural Science Foundation of China (Grant No. 11571293) and Key Research and Development Program of Hunan Province, P. R. China (Grant No. 2017SK2014 ). \newline
		Email addresses: lizheng.math.ai@gmail.com (Zheng Li),	shushi@xtu.edu.cn (Shi Shu)
	}
}
\affil[a]{\textit{School of Mathematics and Computational Science, Xiangtan University, Xiangtan, Hunan, 411105, China}}
\affil[b]{\textit{Hunan Key Laboratory for Computation and Simulation in Science and Engineering, Xiangtan University, Xiangtan 411105, China}}

\begin{document}
\maketitle

\begin{abstract}
Inspired by dynamic programming, we propose Stochastic Virtual Gradient Descent (SVGD) algorithm where the Virtual Gradient is defined by computational graph and automatic differentiation. The method is computationally efficient and has little memory requirements. We also analyze the theoretical convergence properties and implementation of the algorithm. Experimental results on multiple datasets and network models show that SVGD has advantages over other stochastic optimization methods.
\end{abstract}
\section{Introduction} 
\label{sec:Introduction}
Stochastic gradient-based optimization is most widely used in many fields of science and engineering. In recent years, many scholars have compared SGD\cite{saad1998online} with some adaptive learning rate optimization methods\cite{hinton2012neural,kingma2014adam}. \cite{wilson2017marginal} shows that adaptive methods often display faster initial progress on the training set, but their performance quickly plateaus on the development/test set. Therefore, many excellent models \cite{girshick2014rich,ren2015faster,xie2017aggregated} still use SGD for training. However, SGD is greedy for the objective function with many multi-scale local convex regions (cf. Figure 1 of \cite{zhang2017hitting} or Fig. \ref{fig:Sequence}, left) because the negative of the gradient may not point to the minimum point on coarse-scale. Thus, the learning rate of SGD is difficult to set and significantly affects model performance\cite{Goodfellow-et-al-2016}.

Unlike greedy methods, dynamic programming (DP) \cite{cormen2009introduction} converges faster by solving simple sub-problems that decomposed from the original problem. Inspired by this, we propose the \textbf{virtual gradient} to construct a stochastic optimization method that combines the advantages of SGD and adaptive learning rate methods.

Consider a general objective function with the following composite form:
\begin{equation}
J = F(\boldsymbol{\sigma}), \quad \boldsymbol{\sigma} = \boldsymbol{f} (\boldsymbol{\theta}) \in \Omega_{\boldsymbol{\sigma}},
\label{composite-function}
\end{equation}
where $\boldsymbol{\theta} \in \Omega_{\boldsymbol{\theta}} = \mathbb{R}^{n}, \Omega_{\boldsymbol{\sigma}} = \boldsymbol{f}(\Omega_{\boldsymbol{\theta}}) \subseteq \mathbb{R}^{m}$, functions $F$ and each component function of $\boldsymbol{f}$ is first-order differentiable. 

We note that:
\begin{equation}
F(\boldsymbol{\sigma}^{*}) = F(\boldsymbol{f}(\boldsymbol{\theta}^{*})), 
\quad \boldsymbol{\sigma}^* = \mathop{\arg\min}_{\boldsymbol{\sigma} \in \Omega_{\boldsymbol{\sigma}}}\ F(\boldsymbol{\sigma}),
\quad \boldsymbol{\theta}^* = \mathop{\arg\min}_{\boldsymbol{\theta} \in \Omega_{\boldsymbol{\theta}}}\ F(\boldsymbol{f}( \boldsymbol{\theta})).
\end{equation}
In addition, when we minimize $F(\boldsymbol{\sigma})$ and $F(\boldsymbol{f}(\boldsymbol{\theta}))$ with the same iterative method, the former should converge faster because the structure of $F$ is simpler than $F \circ \boldsymbol{f}$. Based on these facts, we construct sequences $\{ \boldsymbol{\sigma}^{(t)} \}$ and $\{ \boldsymbol{\theta}^{(t)} \}$ that converge to $\boldsymbol{\sigma}^{*}$ and $\boldsymbol{\theta}^{*}$, respectively, with equations:
\begin{equation}
\boldsymbol{\sigma}^{(t)} = \boldsymbol{f}(\boldsymbol{\theta}^{(t)}), t=0, 1, \cdots.
\label{condition}
\end{equation}
Fig. \ref{fig:Sequence} (right) shows the relationship between $\{ \boldsymbol{\sigma}^{(t)} \}$ and $\{ \boldsymbol{\theta}^{(t)} \}$. The sequence $\{ \boldsymbol{\sigma}^{(t)} \}$ can be obtained by using first-order iterative methods (see Sec.\ref{sec:Related_Work} for details):
\begin{equation}
\boldsymbol{\sigma}^{(t + 1)} = \boldsymbol{\sigma}^{(t)} -\alpha \mathscr{T}^{*} \nabla_{\boldsymbol{\sigma}} J \big|_{\boldsymbol{\sigma} = \boldsymbol{\sigma}^{(t)}}, 
\label{first-order-formula-for-sigma}
\end{equation}
where $\alpha$ is the learning rate, $\mathscr{T}^*$ is an operator of mappping $\mathbb{R}^m \rightarrow \mathbb{R}^m$. 
\begin{figure}[H]
	\centering
	\includegraphics[width=0.8\textwidth, angle=0]{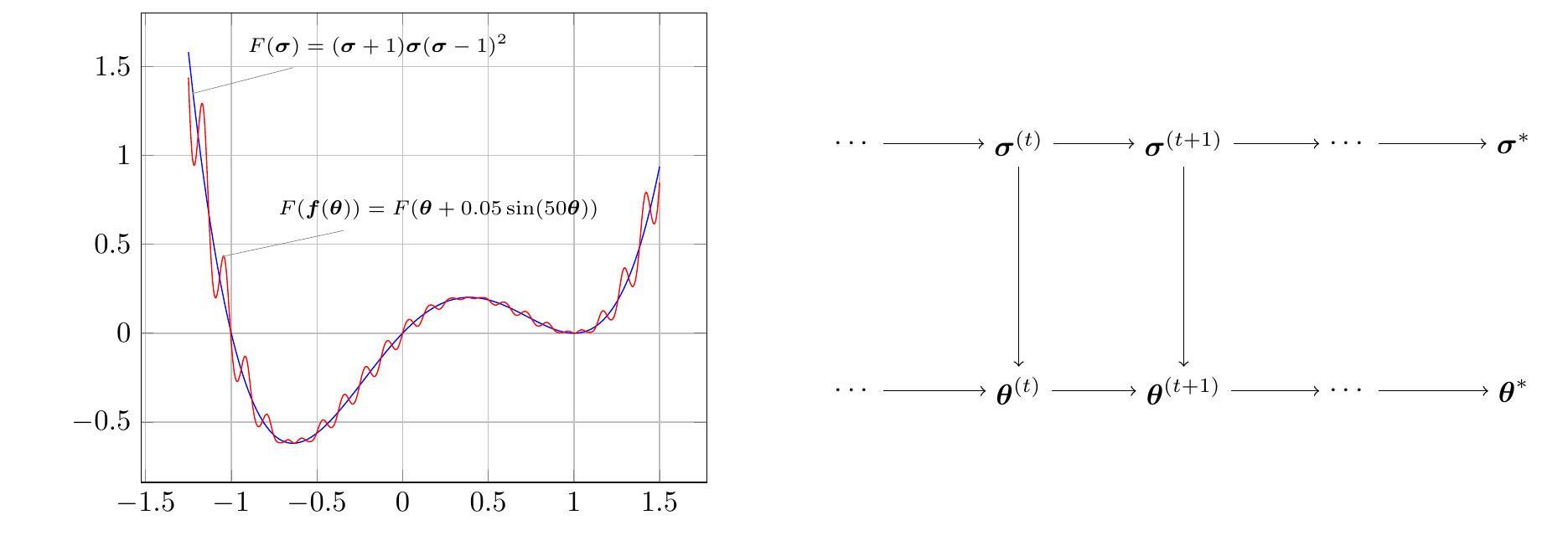}
	\caption{}
	\label{fig:Sequence}
\end{figure}

The difficulty in constructing operator $\mathscr{T}^*$ is how to make the condition (\ref{condition}) holds true. Let $\boldsymbol{M} = \left( \frac{\partial \boldsymbol{f}_i (\boldsymbol{\theta})}{\partial \boldsymbol{\theta}_j} \Big|_{\boldsymbol{\theta} = \boldsymbol{\theta}^{(t)}} \right)_{m \times n}$, $\mathscr{T}$ is an operator of mapping $\mathbb{R}^m \rightarrow \mathbb{R}^m$, we give the following iterations:
\begin{equation}
\boldsymbol{\sigma}^{(t + 1)} = \boldsymbol{\sigma}^{(t)} -\alpha \boldsymbol{M} \boldsymbol{M}^T \mathscr{T} \nabla_{\boldsymbol{\sigma}} J \big|_{\boldsymbol{\sigma} = \boldsymbol{\sigma}^{(t)}} ,
\label{first-order-formula-for-theta-sigma}
\end{equation}
\begin{equation}
\boldsymbol{\theta}^{(t + 1)} = \boldsymbol{\theta}^{(t)} -\alpha \boldsymbol{M}^T \mathscr{T} \nabla_{\boldsymbol{\sigma}} J \big|_{\boldsymbol{\theta} = \boldsymbol{\theta}^{(t)}} .
\label{first-order-formula-for-theta}
\end{equation}
Since $\boldsymbol{M}^T \mathscr{T} \nabla_{\boldsymbol{\sigma}} J$ in Eqn.(\ref{first-order-formula-for-theta}) is equivalent to the position of $\nabla_{\boldsymbol{\theta}} J$ in gradient descent method, we define $\boldsymbol{M}^T \mathscr{T} \nabla_{\boldsymbol{\sigma}} J$ as the \textbf{virtual gradient} of function $J$ for variable $\boldsymbol{\theta}$.

For Eqn.(\ref{first-order-formula-for-theta}), it is easy to prove that the condition (\ref{condition}) holds when $\boldsymbol{f}$ is a linear mapping. If $\boldsymbol{f}$ is a nonlinear mapping, let the second-derivatives of $\boldsymbol{f}$ be bounded and $\alpha = o(1), \boldsymbol{\sigma}^{(t)} = \boldsymbol{f}(\boldsymbol{\theta}^{(t)})$, owing to (\ref{first-order-formula-for-theta-sigma}) and (\ref{first-order-formula-for-theta}) and Taylor formula, the following holds true:
\begin{align}
\begin{autobreak}
|| \boldsymbol{\sigma}^{(t+1)} - \boldsymbol{f}(\boldsymbol{\theta}^{(t+1)}) || 
= || (\boldsymbol{\sigma}^{(t)} -\alpha \boldsymbol{M} \boldsymbol{M}^T \mathscr{T} \nabla_{\boldsymbol{\sigma}} J \big|_{\boldsymbol{\sigma} = \boldsymbol{\sigma}^{(t)}}) - \boldsymbol{f}(\boldsymbol{\theta}^{(t)} -\alpha \boldsymbol{M}^T \mathscr{T} \nabla_{\boldsymbol{\sigma}} J \big|_{\boldsymbol{\theta} = \boldsymbol{\theta}^{(t)}}) || 
= O(||\alpha \boldsymbol{M}^T \mathscr{T} \nabla_{\boldsymbol{\sigma}} J \big|_{\boldsymbol{\theta} = \boldsymbol{\theta}^{(t)}} ||_2^2) 
= O(\alpha^2),
\end{autobreak}
\label{taylor}
\end{align}
In this case, the condition (\ref{condition}) holds, approximately.

According to the analysis above, the sequence $\{F(\boldsymbol{f}(\boldsymbol{\theta}^{(t)}))\}$ yields similar convergence as $\{F(\boldsymbol{\sigma}^{(t)})\}$ in Eqn.(\ref{first-order-formula-for-theta}) and Eqn.(\ref{first-order-formula-for-theta-sigma}), but faster than minimizing the function $F(\boldsymbol{f}(\boldsymbol{\theta}))$ with the same first-order method, directly.

Note that the iterative method (\ref{first-order-formula-for-theta}) is derived based on the composite form (\ref{composite-function}) and this form is generally not unique, it is inconvenient for our algorithm design. We begin by introducing the computational graph. It is a directed graph, where each node indicates a variable that may be a scalar, vector, matrix, tensor, or even a variable of another type, and each edge unique corresponds to an operation which maps a node to another. We sometimes annotate the output node with the name of the operation applied. In particular, the computational graph corresponding to the objective function is a DAG(directed acyclic graphs) \cite{thulasiraman1992graphs}. For example, the computational graph of the objective function $J$ shown in Fig. \ref{fig:Computational-Graphs-Forward} (a), the corresponding composite form (\ref{composite-function}) is:
\begin{align}
\left\{
\begin{array}{l}
J = F(\boldsymbol{\sigma}), \boldsymbol{\sigma} 
= \begin{bmatrix}\tilde{\boldsymbol{\sigma}}_1 \\ \tilde{\boldsymbol{\sigma}}_2 \end{bmatrix}
= \boldsymbol{f}(\boldsymbol{\theta})
= \begin{bmatrix}
\tilde{\boldsymbol{f}}_1(\tilde{\boldsymbol{\theta}}_1, \tilde{\boldsymbol{\theta}}_2) \\
\tilde{\boldsymbol{f}}_2(\tilde{\boldsymbol{\theta}}_2, \tilde{\boldsymbol{\theta}}_3; \tilde{\boldsymbol{\sigma}}_1)
\end{bmatrix}, 
\quad \boldsymbol{\theta} = [\tilde{\boldsymbol{\theta}}_1^T, \tilde{\boldsymbol{\theta}}_2^T, \tilde{\boldsymbol{\theta}}_3^T]^T \\
\tilde{\boldsymbol{f}_1}: \mathbb{R}^{n_1^{\boldsymbol{\theta}}} \times \mathbb{R}^{n_2^{\boldsymbol{\theta}}} \rightarrow \mathbb{R}^{n_1^{\boldsymbol{\sigma}}}, \quad \tilde{\boldsymbol{f}_2}: \mathbb{R}^{n_2^{\boldsymbol{\theta}}} \times \mathbb{R}^{n_3^{\boldsymbol{\theta}}} \times \mathbb{R}^{n_1^{\boldsymbol{\sigma}}} \rightarrow \mathbb{R}^{n_2^{\boldsymbol{\sigma}}} \\
\end{array}
\ \right. . \label{composite-function-instance}
\end{align} 
\begin{figure}[H]
	\centering
	\includegraphics[width=0.9\textwidth, angle=0]{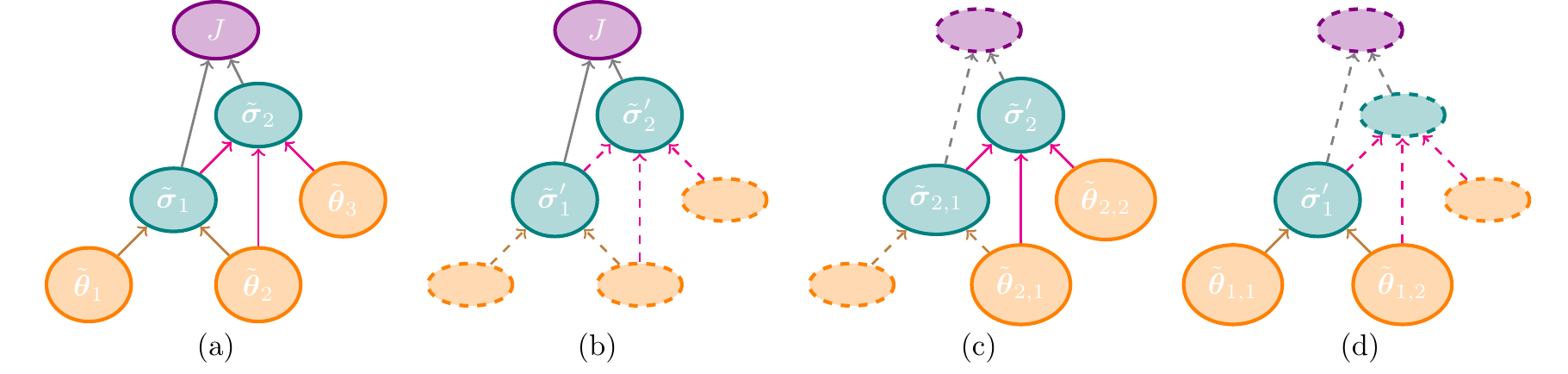}
	\caption{
		\protect\tikz \protect\draw (0,0) node[ellipse,minimum height=2mm,minimum width=5mm,draw=orange!100,fill=orange!30,very thick] { };,
		\protect\tikz \protect\draw (0,0) node[ellipse,minimum height=2mm,minimum width=5mm,draw=teal!100,fill=teal!30,very thick] { };,
		\protect\tikz \protect\draw (0,0) node[ellipse,minimum height=2mm,minimum width=5mm,draw=violet!100,fill=violet!30,very thick] { };
		are nodes associated with leaf values, hidden values and output value, and  
		\protect\tikz \protect\draw[brown,thick,->](0,0) -- (5mm,0); ,
		\protect\tikz \protect\draw[magenta,thick,->](0,0) -- (5mm,0); , 
		\protect\tikz \protect\draw[gray,thick,->](0,0) -- (5mm,0); 
		are edges associated with operations $\tilde{\boldsymbol{f}}_1, \tilde{\boldsymbol{f}}_2, F$. 
	}
	\label{fig:Computational-Graphs-Forward}
\end{figure}
For a given general objective function, let $\boldsymbol{G}$ correspond to a computational graph that maps the set of leaf values $\boldsymbol{V}_{\boldsymbol{\theta}}^{\boldsymbol{G}} = \{ \tilde{\boldsymbol{\theta}}_j | j =1, \cdots, N\}$ to the output value $J$, where the set of hidden values is $\boldsymbol{V}_{\boldsymbol{\sigma}}^{\boldsymbol{G}} = \{ \tilde{\boldsymbol{\sigma}}_i | i =1, \cdots, M\}$. Let $\boldsymbol{V}_{\boldsymbol{\theta} \rightarrow \boldsymbol{\sigma}}^{\boldsymbol{G}} := \{ \tilde{\boldsymbol{\sigma}}'_i | i =1, \cdots, M'\} = \{ \tilde{\boldsymbol{\sigma}}_i \in \boldsymbol{V}_{\boldsymbol{\sigma}}^{\boldsymbol{G}} |  \mathrm{dist}(\tilde{\boldsymbol{\sigma}}_i, \boldsymbol{V}_{\boldsymbol{\theta}}^{\boldsymbol{G}}) = 1 \}$. In this paper, the objective function $J$ in Eqn.(\ref{composite-function}) will be expressed as the following composite form:
\begin{equation}
J = F(\boldsymbol{\sigma}), \boldsymbol{\sigma} 
= \begin{bmatrix}\tilde{\boldsymbol{\sigma}}'_1 \\ \vdots \\ \tilde{\boldsymbol{\sigma}}'_{M'}\end{bmatrix}
= \tilde{\boldsymbol{f}}(\boldsymbol{\theta})=\begin{bmatrix}
\tilde{\boldsymbol{f}}_1(\tilde{\boldsymbol{\theta}}_{1,1}, \cdots, \tilde{\boldsymbol{\theta}}_{1,N_1}; \tilde{\boldsymbol{\sigma}}'_{1,1}, \cdots, \tilde{\boldsymbol{\sigma}}'_{1,N'_1}) \\
\vdots \\
\tilde{\boldsymbol{f}}_{M'}(\tilde{\boldsymbol{\theta}}_{M',1}, \cdots, \tilde{\boldsymbol{\theta}}_{M',N_{M'}}; \tilde{\boldsymbol{\sigma}}'_{M',1}, \cdots, \tilde{\boldsymbol{\sigma}}'_{M',N'_{M'}})
\end{bmatrix}
\label{graph-decomposition}
\end{equation}
where $\tilde{\boldsymbol{\sigma}}'_i = \tilde{\boldsymbol{f}}_i(\tilde{\boldsymbol{\theta}}_{i,1}, \cdots, \tilde{\boldsymbol{\theta}}_{i,N_i}; \tilde{\boldsymbol{\sigma}}'_{i,1}, \cdots, \tilde{\boldsymbol{\sigma}}'_{i,N'_i}) \in \boldsymbol{V}_{\boldsymbol{\theta} \rightarrow \boldsymbol{\sigma}}^{\boldsymbol{G}}$, $i=1, \cdots, M'$, and 
\begin{equation*}
\left\{
\begin{array}{l}
\{ \tilde{\boldsymbol{\theta}}_{i,1}, \cdots, \tilde{\boldsymbol{\theta}}_{i,N_i} \} = \{ \tilde{\boldsymbol{\theta}}_k \in \boldsymbol{V}_{\boldsymbol{\theta}}^{\boldsymbol{G}} |\ \mathrm{dist}(\tilde{\boldsymbol{\theta}}_k, \tilde{\boldsymbol{\sigma}}'_i) = 1 \} \\
\{ \tilde{\boldsymbol{\sigma}}'_{i,1}, \cdots, \tilde{\boldsymbol{\sigma}}'_{i,N'_i} \} = \{ \tilde{\boldsymbol{\sigma}}_k \in \boldsymbol{V}_{\boldsymbol{\sigma}}^{\boldsymbol{G}} |\ \mathrm{dist}(\tilde{\boldsymbol{\sigma}}_k, \tilde{\boldsymbol{\sigma}}'_i) = 1, \tilde{\boldsymbol{\sigma}}_k \preceq \tilde{\boldsymbol{\sigma}}'_i \}
\end{array}
\ \right. .
\end{equation*}
For example, Eqn.(\ref{composite-function-instance}) can be expressed as:
\begin{align*}
J = F(\boldsymbol{\sigma}), \boldsymbol{\sigma} 
= \begin{bmatrix}\tilde{\boldsymbol{\sigma}}'_1 \\ \tilde{\boldsymbol{\sigma}}'_2 \end{bmatrix}
= \tilde{\boldsymbol{f}}(\boldsymbol{\theta})
= \begin{bmatrix}
\tilde{\boldsymbol{f}}_1(\tilde{\boldsymbol{\theta}}_{1,1}, \tilde{\boldsymbol{\theta}}_{1,2}) \\
\tilde{\boldsymbol{f}}_2(\tilde{\boldsymbol{\theta}}_{2,1}, \tilde{\boldsymbol{\theta}}_{2,2}; \tilde{\boldsymbol{\sigma}}_{2,1})
\end{bmatrix},
\end{align*} 
where 
\begin{align*}
\left\{
\begin{array}{l}
\tilde{\boldsymbol{\sigma}}'_1 = \tilde{\boldsymbol{\sigma}}_1, \tilde{\boldsymbol{\theta}}_{1,1} = \tilde{\boldsymbol{\theta}}_1,
\tilde{\boldsymbol{\theta}}_{1,2} = \tilde{\boldsymbol{\theta}}_2 \\
\tilde{\boldsymbol{\sigma}}'_2 = \tilde{\boldsymbol{\sigma}}_2,
\tilde{\boldsymbol{\sigma}}_{2,1} = \tilde{\boldsymbol{\sigma}}_1, \tilde{\boldsymbol{\theta}}_{2,1} = \tilde{\boldsymbol{\theta}}_2,
\tilde{\boldsymbol{\theta}}_{2,2} = \tilde{\boldsymbol{\theta}}_3
\end{array}
\ \right. . 
\end{align*} 
In deeping learning, the gradient of the objective function is usually calculated by the Automatic Differentiation (AD) technique\cite{baydin2018automatic,Goodfellow-et-al-2016}. Our following example introduces how to calculate the gradient of $J$ in Eqn.(\ref{composite-function-instance}) using AD technique.
\begin{enumerate}
	\item Find the Operation $F$ associated with output value $J$ and its input node $\{\tilde{\boldsymbol{\sigma}}'_1, \tilde{\boldsymbol{\sigma}}'_2 \}$, cf. Fig. \ref{fig:Computational-Graphs-Forward} (b). Then, calculate the following gradients:
	\begin{equation*}
		\boldsymbol{g}_{\boldsymbol{w}}^{J}
		:= \boldsymbol{g}_{\boldsymbol{w} \rightarrow J} (\tilde{\boldsymbol{\sigma}}'_1, \tilde{\boldsymbol{\sigma}}'_2)
		= \nabla_{\boldsymbol{w}} F, \quad \boldsymbol{w} = \tilde{\boldsymbol{\sigma}}'_1, \tilde{\boldsymbol{\sigma}}'_2.
	\end{equation*}
	\item Perform the following steps by the partial order $\succeq$ of $\{\tilde{\boldsymbol{\sigma}}'_1, \tilde{\boldsymbol{\sigma}}'_2 \}$:
	\begin{enumerate}
		\item Find Operation $\tilde{\boldsymbol{f}}_2$ and it's input nodes $\{\tilde{\boldsymbol{\theta}}_{2,1}, \tilde{\boldsymbol{\theta}}_{2,2}, \tilde{\boldsymbol{\sigma}}'_{2,1} \}$ which associated with hidden value $\tilde{\boldsymbol{\sigma}}'_2$, cf. Fig. \ref{fig:Computational-Graphs-Forward} (c). Let:
		\begin{equation*}
			F_2(\tilde{\boldsymbol{\theta}}_{2,1}, \tilde{\boldsymbol{\theta}}_{2,2}; \tilde{\boldsymbol{\sigma}}'_{2,1}) := F(\cdot, \tilde{\boldsymbol{f}}_2(\tilde{\boldsymbol{\theta}}_{2,1}, \tilde{\boldsymbol{\theta}}_{2,2}; \tilde{\boldsymbol{\sigma}}'_{2,1})),
		\end{equation*}
		where '$\cdot$' denotes that it is treated as a constant during the calculation of gradients and will not be declared later. Calculate the following gradients:
		\begin{equation*}
			\boldsymbol{g}_{\boldsymbol{w}}^{\tilde{\boldsymbol{\sigma}}'_2} := \boldsymbol{g}_{\boldsymbol{w} \rightarrow \tilde{\boldsymbol{\sigma}}'_2}(\nabla_{\tilde{\boldsymbol{\sigma}}'_2} J; \tilde{\boldsymbol{\theta}}_{2,1}, \tilde{\boldsymbol{\theta}}_{2,2}; \tilde{\boldsymbol{\sigma}}'_{2,1}) = \nabla_{\boldsymbol{w}} F_2 =
			\left( \frac{\partial \tilde{\boldsymbol{\sigma}}'_2}{\partial \boldsymbol{w}} \right)^T \nabla_{\tilde{\boldsymbol{\sigma}}'_2} J, \quad \boldsymbol{w} = \tilde{\boldsymbol{\theta}}_{2,1}, \tilde{\boldsymbol{\theta}}_{2,2}, \tilde{\boldsymbol{\sigma}}'_{2,1},
		\end{equation*}
		where $\nabla_{\tilde{\boldsymbol{\sigma}}'_2} J = \nabla_{\tilde{\boldsymbol{\sigma}}'_2} F$.
		\item Find Operation $\tilde{\boldsymbol{f}}_1$ and it's input nodes $\{\tilde{\boldsymbol{\theta}}_{1,1}, \tilde{\boldsymbol{\theta}}_{1,2} \}$ which associated with hidden value $\tilde{\boldsymbol{\sigma}}'_1$, cf. Fig. \ref{fig:Computational-Graphs-Forward} (d). Let:
		\begin{equation*}
			F_1(\tilde{\boldsymbol{\theta}}_{1,1}, \tilde{\boldsymbol{\theta}}_{1,2}) := F(\tilde{\boldsymbol{f}}_1(\tilde{\boldsymbol{\theta}}_{1,1}, \tilde{\boldsymbol{\theta}}_{1,2}), \tilde{\boldsymbol{f}}_2(\cdot, \cdot, \tilde{\boldsymbol{f}}_1(\tilde{\boldsymbol{\theta}}_{1,1}, \tilde{\boldsymbol{\theta}}_{1,2}))).
		\end{equation*}
		Calculate following gradients:
		\begin{equation*}
			\boldsymbol{g}_{\boldsymbol{w}}^{\tilde{\boldsymbol{\sigma}}'_1} := \boldsymbol{g}_{\boldsymbol{w} \rightarrow \tilde{\boldsymbol{\sigma}}'_1}(\nabla_{\tilde{\boldsymbol{\sigma}}'_1} J; \tilde{\boldsymbol{\theta}}_{1,1}, \tilde{\boldsymbol{\theta}}_{1,2}) = \nabla_{\boldsymbol{w}} F_1 =
			\left( \frac{\partial \tilde{\boldsymbol{\sigma}}'_1}{\partial \boldsymbol{w}} \right)^T \nabla_{\tilde{\boldsymbol{\sigma}}'_1} J, \quad \boldsymbol{w} = \tilde{\boldsymbol{\theta}}_{1,1}, \tilde{\boldsymbol{\theta}}_{1,2},
		\end{equation*}
		where $\nabla_{\tilde{\boldsymbol{\sigma}}'_1} J = \nabla_{\tilde{\boldsymbol{\sigma}}'_1} F + \nabla_{\tilde{\boldsymbol{\sigma}}'_1} F_2$.
	\end{enumerate}
	\item Calculate the gradients of $J$:
	\begin{equation*}
		\nabla_{\tilde{\boldsymbol{\theta}}_1} J = \boldsymbol{g}_{\tilde{\boldsymbol{\theta}}_{1,1}}^{\tilde{\boldsymbol{\sigma}}'_1}, \nabla_{\tilde{\boldsymbol{\theta}}_2} J = \boldsymbol{g}_{\tilde{\boldsymbol{\theta}}_{1,2}}^{\tilde{\boldsymbol{\sigma}}'_1} + \boldsymbol{g}_{\tilde{\boldsymbol{\theta}}_{2,1}}^{\tilde{\boldsymbol{\sigma}}'_2},
		\nabla_{\tilde{\boldsymbol{\theta}}_3} J = \boldsymbol{g}_{\tilde{\boldsymbol{\theta}}_{2,2}}^{\tilde{\boldsymbol{\sigma}}'_2}.
	\end{equation*}
\end{enumerate}
According to the analysis above, the computational graph of $\{\nabla_{\tilde{\boldsymbol{\theta}}_k} J|\ k=1, 2, 3\}$ can be shown as Fig. \ref{fig:Computational-Graphs-Backward} (a). If $\mathscr{T}$ is a broadcast-like operator, the computational graph of \textbf{vitrual gradients} can be shown as Fig. \ref{fig:Computational-Graphs-Backward} (b), where $z_1 = \boldsymbol{g}_{\tilde{\boldsymbol{\theta}}_2 \rightarrow \tilde{\boldsymbol{\sigma}}_1} (\mathscr{T} \nabla_{\tilde{\boldsymbol{\sigma}}_1} J; \tilde{\boldsymbol{\theta}}_1, \tilde{\boldsymbol{\theta}}_2) $, $
z_2 = \boldsymbol{g}_{\tilde{\boldsymbol{\theta}}_2 \rightarrow \tilde{\boldsymbol{\sigma}}_2} (\mathscr{T} \nabla_{\tilde{\boldsymbol{\sigma}}_2} J; \tilde{\boldsymbol{\theta}}_2, \tilde{\boldsymbol{\theta}}_3;\tilde{\boldsymbol{\sigma}}_1)$ and  $\{\nabla_{\tilde{\boldsymbol{\theta}}_k}^{(\boldsymbol{G}, \mathscr{T})} J|\ k=1, 2, 3\}$ is defined by the Eqn.(\ref{virtual-gradient}).
\begin{figure}[H]
	\centering
	\includegraphics[width=0.8\textwidth, angle=0]{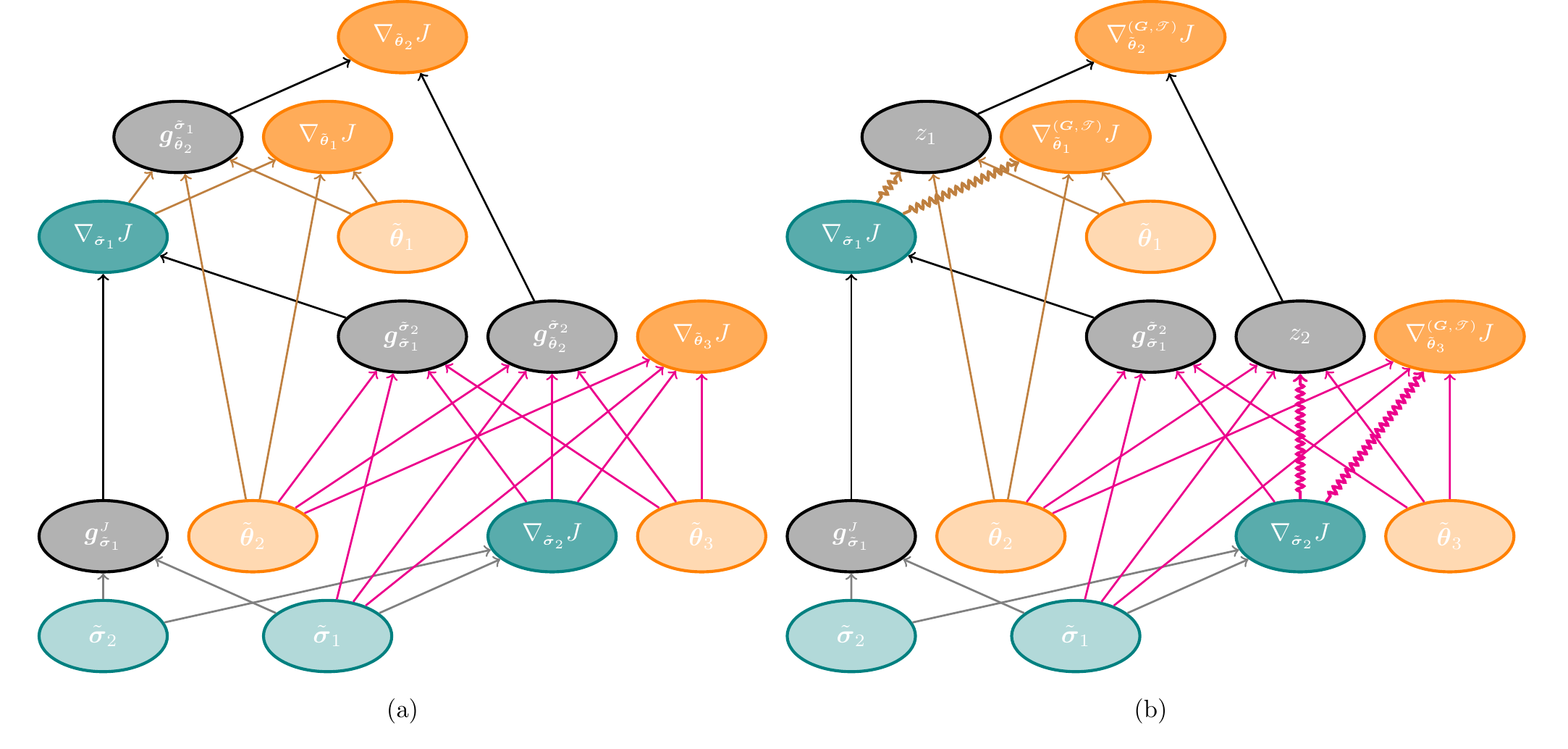}
	\caption{
		\protect\tikz \protect\draw[black,thick,->](0,0) -- (5mm,0); are edges associated with operation $+$ and 
		$x$ \protect\tikz \protect\draw[gray,very thick,decorate,decoration={pre length=1mm,post length=1mm, snake, amplitude=.5mm,segment length=1mm},->](0, 0)--(5mm, 0); $y$ denotes $\mathscr{T} x$ \protect\tikz \protect\draw[gray,very thick,->](0,0) -- (5mm,0); $y$.
	}
	\label{fig:Computational-Graphs-Backward}
\end{figure}
According to the definition of \textbf{virtual gradient}, for any $\tilde{\boldsymbol{\theta}}_k \in \boldsymbol{V}_{\boldsymbol{\theta}}^{\boldsymbol{G}}$:
\begin{align}
\nabla_{\tilde{\boldsymbol{\theta}}_k}^{(\boldsymbol{G}, \mathscr{T})} J 
&= \sum_{\tilde{\boldsymbol{\sigma}}'_i \in \boldsymbol{V}_{\boldsymbol{\theta} \rightarrow \boldsymbol{\sigma}}} \left( \frac{\partial \tilde{\boldsymbol{\sigma}}'_i}{\partial \tilde{\boldsymbol{\theta}}_k} \right)^T \mathscr{T}_{i} \nabla_{\tilde{\boldsymbol{\sigma}}'_i} J \nonumber \\
&= \sum_{\tilde{\boldsymbol{\sigma}}'_i \in \boldsymbol{V}_{\boldsymbol{\theta} \rightarrow \boldsymbol{\sigma}}} \sum_j \delta(\tilde{\boldsymbol{\theta}}_{i,j}, \tilde{\boldsymbol{\theta}}_k) \boldsymbol{g}_{\tilde{\boldsymbol{\theta}}_{i,j} \rightarrow \tilde{\boldsymbol{\sigma}}'_i} (\mathscr{T}_{i,j} \nabla_{\tilde{\boldsymbol{\sigma}}'_i} J; \tilde{\boldsymbol{\theta}}_{i,1}, \cdots, \tilde{\boldsymbol{\theta}}_{i,N_1}; \tilde{\boldsymbol{\sigma}}'_{i,1}, \cdots, \tilde{\boldsymbol{\sigma}}'_{i,N'_i}).
\label{virtual-gradient}
\end{align}
Obviously, $\nabla_{\tilde{\boldsymbol{\theta}}_k} J = \nabla_{\tilde{\boldsymbol{\theta}}_k}^{(\boldsymbol{G}, I)} J$ where $I$ is an identity operator. The \textbf{bprop} operation $\boldsymbol{g}_{\boldsymbol{w} \rightarrow \tilde{\boldsymbol{\sigma}}'_i}$ is uniquely determined by $\tilde{\boldsymbol{f}}_i$.

Then, the Eqn.(\ref{first-order-formula-for-theta}) can be written as the following virtual gradient descent iteration:
\begin{equation}
	\quad \tilde{\boldsymbol{\theta}}_k^{(t+1)} = \quad \tilde{\boldsymbol{\theta}}_k^{(t)} - \alpha \nabla_{\tilde{\boldsymbol{\theta}}_k}^{(\boldsymbol{G}, \mathscr{T})} J \big|_{\boldsymbol{\theta} = \boldsymbol{\theta}^{(t)}}, \quad \forall \tilde{\boldsymbol{\theta}}_k \in \boldsymbol{V}_{\boldsymbol{\theta}}^{\boldsymbol{G}}
	\label{virtual-gradient-descent}
\end{equation}
We prove that the SVGD (Alg. \ref{alg:SVGD}) has advantages over SGD, RMSProp and Adam in training speed and test accuracy by experiments on multiple network models and datasets.

In Sec.\ref{sec:Stochastic_Virtual_Gradients_Descent_Method} we describe the operator $\mathscr{T}$ and the SVGD algorithm of stochastic optimization. Sec.\ref{sec:Encapsulation} introduce two methods to encapsulate SVGD, and Sec.\ref{sec:Convergence_Analysis} provides a theoretical analysis of convergence. Sec.\ref{sec:Experiments} compares our method with other methods by experiments.
\section{Stochastic Virtual Gradients Descent Method}
\label{sec:Stochastic_Virtual_Gradients_Descent_Method}
In this section, we will use the \textbf{accumulate squared gradient} in the RMSProp to construct the operator $\mathscr{T}$. According to Eqn.(\ref{taylor}), Eqn.(\ref{condition}) holds when the mapping $\boldsymbol{f}$ is linear. Based on this fact, we designed the following SVGD algorithm. The functions and variables in the algorithm are given by Eqn.(\ref{graph-decomposition}) and Eqn.(\ref{virtual-gradient}).

\begin{algorithm}[H]
	\DontPrintSemicolon
	\SetAlgoLined
	\SetNoFillComment
	\LinesNotNumbered
	
	\KwInput{$\boldsymbol{G}$: computational graph associated with function $J^{(\tau)} = L(\hat{y} (\boldsymbol{x}^{(\tau)}, \boldsymbol{\theta}), y^{(\tau)})$}
	\KwInput{$\alpha$: Learning rate}
	\KwInput{$m$: Minibatch size}
	\KwInput{$s \in [0, +\infty)$: Scaling coefficient}
	\KwInput{$\boldsymbol{\theta}$: Initial parameter}
	
	\tcc{define operator $\nabla^{(\boldsymbol{G}, \mathscr{T})}$ before training}
	\For{$\tilde{\boldsymbol{\sigma}}'_i \in \boldsymbol{V}_{\boldsymbol{\theta} \rightarrow \boldsymbol{\sigma}}^{\boldsymbol{G}} $}    
	{ 
		$\boldsymbol{r}_i = \boldsymbol{0}$ 
		\tcp*[l]{Initialize gradient accumulation variable}
		\For{$j \in \{ 1, \cdots, N_i \} $}
		{
			\eIf{$\tilde{\boldsymbol{f}}_i$ about $\tilde{\boldsymbol{\theta}}'_{i,j}$ is linear}
			{
				$\boldsymbol{g}_{\tilde{\boldsymbol{\theta}}'_{i,j} \rightarrow \tilde{\boldsymbol{\sigma}}'_i} (\mathscr{T}_{i,j} \nabla_{\tilde{\boldsymbol{\sigma}}'_i} J^{(\tau)}; \cdots; \cdots) 
				= \boldsymbol{g}_{\tilde{\boldsymbol{\theta}}'_{i,j} \rightarrow \tilde{\boldsymbol{\sigma}}'_i} (s \frac{\nabla_{\tilde{\boldsymbol{\sigma}}'_i} J^{(\tau)}}{\sqrt{\boldsymbol{r}_i + \epsilon}}; \cdots; \cdots)$ 
				\tcp*[l]{define $\mathscr{T}_{i,j}$}
			}{
				$\boldsymbol{g}_{\tilde{\boldsymbol{\theta}}'_{i,j} \rightarrow \tilde{\boldsymbol{\sigma}}'_i} (\mathscr{T}_{i,j} \nabla_{\tilde{\boldsymbol{\sigma}}'_i} J^{(\tau)}; \cdots; \cdots) 
				= \boldsymbol{g}_{\tilde{\boldsymbol{\theta}}'_{i,j} \rightarrow \tilde{\boldsymbol{\sigma}}'_i} (\nabla_{\tilde{\boldsymbol{\sigma}}'_i} J^{(\tau)}; \cdots; \cdots)$ 
				\tcp*[l]{define $\mathscr{T}_{i,j}$}
			}
		}
	}
	
	\tcc{update $\boldsymbol{\theta}$}
	\While{$\boldsymbol{V}_{\boldsymbol{\theta}^{(t)}}^{\boldsymbol{G}}$ not converged}
	{
		Sample a minibatch of $m$ examples from the training set $\{ \boldsymbol{x}^{(1)}, \cdots, \boldsymbol{x}^{(m)} \}$ with corresponding targets $y^{(i)}$. \\
		\For{$\tilde{\boldsymbol{\sigma}}'_i \in \boldsymbol{V}_{\boldsymbol{\theta} \rightarrow \boldsymbol{\sigma}}^{\boldsymbol{G}} $ parallel} 
		{
			$\boldsymbol{r}_i \leftarrow \rho \boldsymbol{r}_i + (1 - \rho) \frac{1}{m} \sum\limits_{i=\tau}^{m} \left( \nabla_{\tilde{\boldsymbol{\sigma}}'_i} J^{(\tau)}  \right)^2 $ 
			\tcp*[l]{Accumulate squared gradient}
		}
		\For{$\tilde{\boldsymbol{\theta}}_j \in \boldsymbol{V}_{\boldsymbol{\theta}}^{\boldsymbol{G}}$ parallel} 
		{
			$\tilde{\boldsymbol{\theta}}_j \leftarrow \tilde{\boldsymbol{\theta}}_j - \alpha \frac{1}{m} \sum\limits_{i=\tau}^{m} \nabla_{\tilde{\boldsymbol{\theta}}_j}^{(\boldsymbol{G}, \mathscr{T})} J^{(\tau)} $ 
			\tcp*[l]{apply update}
		}
	}
	
	\caption{\textit{SVGD}, our proposed algorithm for stochastic optimization. $(\nabla J)^2$ indicates the elementwise square $\nabla J \odot \nabla J$. Good default settings for the tested machine learning problems are $\epsilon = 10^{-6}$ and $\rho = 0.9$. All operations are element-wise.}
	\label{alg:SVGD}
\end{algorithm}
SVGD works well in neural network training tasks (Fig.\ref{fig:MNIST-MLP}, \ref{fig:CIFAR10-VGG}, \ref{fig:CIFAR10-ResNet}), it has a relatively faster convergence rate and better test accuracy than SGD, RMSProp, and Adam. 

For the linear operation Conv2D \cite{lecun1989generalization} and matrix multiplication MatMul as follows:
\begin{equation*}
\left\{
\begin{array}{l}
\mathrm{Conv2D}: (\mathbb{R}^{N} \times \mathbb{R}^{H_{in}} \times \mathbb{R}^{W_{in}} \times \mathbb{R}^{C_{in}}, \mathbb{R}^{H_k} \times \mathbb{R}^{W_k} \times \mathbb{R}^{C_{in}} \times \mathbb{R}^{C_{out}}) \rightarrow \mathbb{R}^{N} \times \mathbb{R}^{H_{out}} \times \mathbb{R}^{W_{out}} \times \mathbb{R}^{C_{out}} \\
\mathrm{MatMul}: (\mathbb{R}^{N} \times \mathbb{R}^{C_{in}}, \mathbb{R}^{C_{in}} \times \mathbb{R}^{C_{out}}) \rightarrow \mathbb{R}^{N} \times \mathbb{R}^{C_{out}}
\end{array}
\ \right. ,
\end{equation*}
there are $\mathrm{Dim}(\boldsymbol{r}_{\mathrm{Conv2D}}) = H_{out} W_{out} C_{out} < H_k W_k C_{in} C_{out}$ and $\mathrm{Dim}(\boldsymbol{r}_{\mathrm{MatMul}}) = C_{out} < C_{in} C_{out}$. Thus, SVGD also has less memory requirements than RMSProp and Adam for deep neural networks. 

For the same stochastic objective function, the learning rate at timestep t in SVGD has the following relationship with the stepsize in the SGD and RMSProp:
\begin{equation*}
\alpha_{SVGD}(t) \approx \alpha_{SGD}(t), \quad s * \alpha_{SVGD}(t) \approx \alpha_{RMSProp}(t).
\end{equation*}
\section{Encapsulation}
\label{sec:Encapsulation}
In this section, we introduce two methods to generate the computational graph of virtual gradient. We begin by assumming that the objective function is $J$ (cf. Fig. \ref{fig:Encapsulation} (b)), the set  $\boldsymbol{V}_{\boldsymbol{\theta}\rightarrow\boldsymbol{\sigma}}^{\boldsymbol{G}} = \{ \tilde{\boldsymbol{\sigma}}'_1, \tilde{\boldsymbol{\sigma}}'_2 \}$ (cf. Fig. \ref{fig:Encapsulation} (a)), and the function used to construct the computational graph of gradients is "gradients", cf. Fig. \ref{fig:Encapsulation} (c).
\begin{figure}[H]
	\centering
	\includegraphics[width=0.8\textwidth, angle=0]{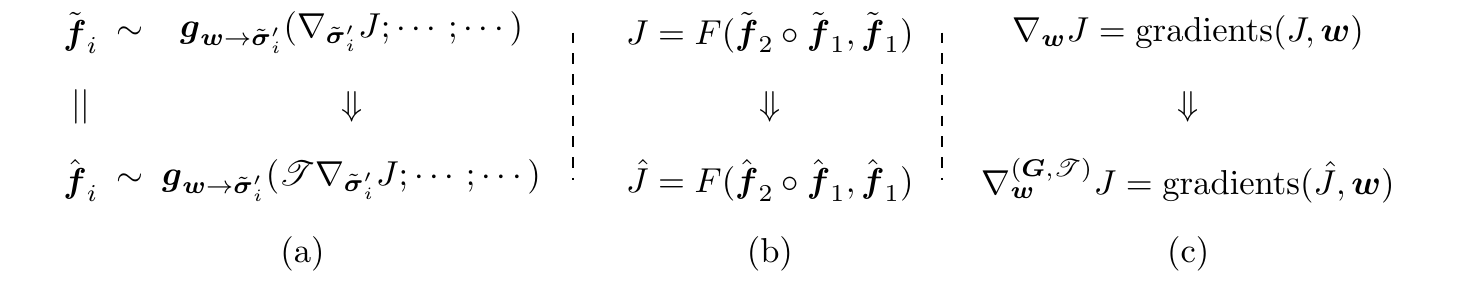}
	\caption{ }
	\label{fig:Encapsulation}
\end{figure}
We hope to generate the computational graph of virtual gradients by using the function "gradients", Fig. \ref{fig:Encapsulation} (c).
\subsection{Extend the API libraries}
As shown in Fig. \ref{fig:Encapsulation}, We begin by replacing $\tilde{\boldsymbol{f}}_i$ with $\hat{\boldsymbol{f}}_i$, where $\hat{\boldsymbol{f}}_i$ is a copy of $\tilde{\boldsymbol{f}}_i$ but corresponds to a new \textbf{bprob} operation. Then, call the function "gradients" to generate the computational graph of virtual gradients.

In order to achieve the idea above, in programming, we need to extending core libraries to customize new operations of $\hat{\boldsymbol{f}}_i$ and its \textbf{bprop} operation. Fig. \ref{fig:Encapsulation-1} shows that we need to extend 3 libraries in the layered architecture of TensorFlow \cite{abadi2016tensorflow}.
\begin{figure}[H]
	\centering
	\includegraphics[width=0.8\textwidth, angle=0]{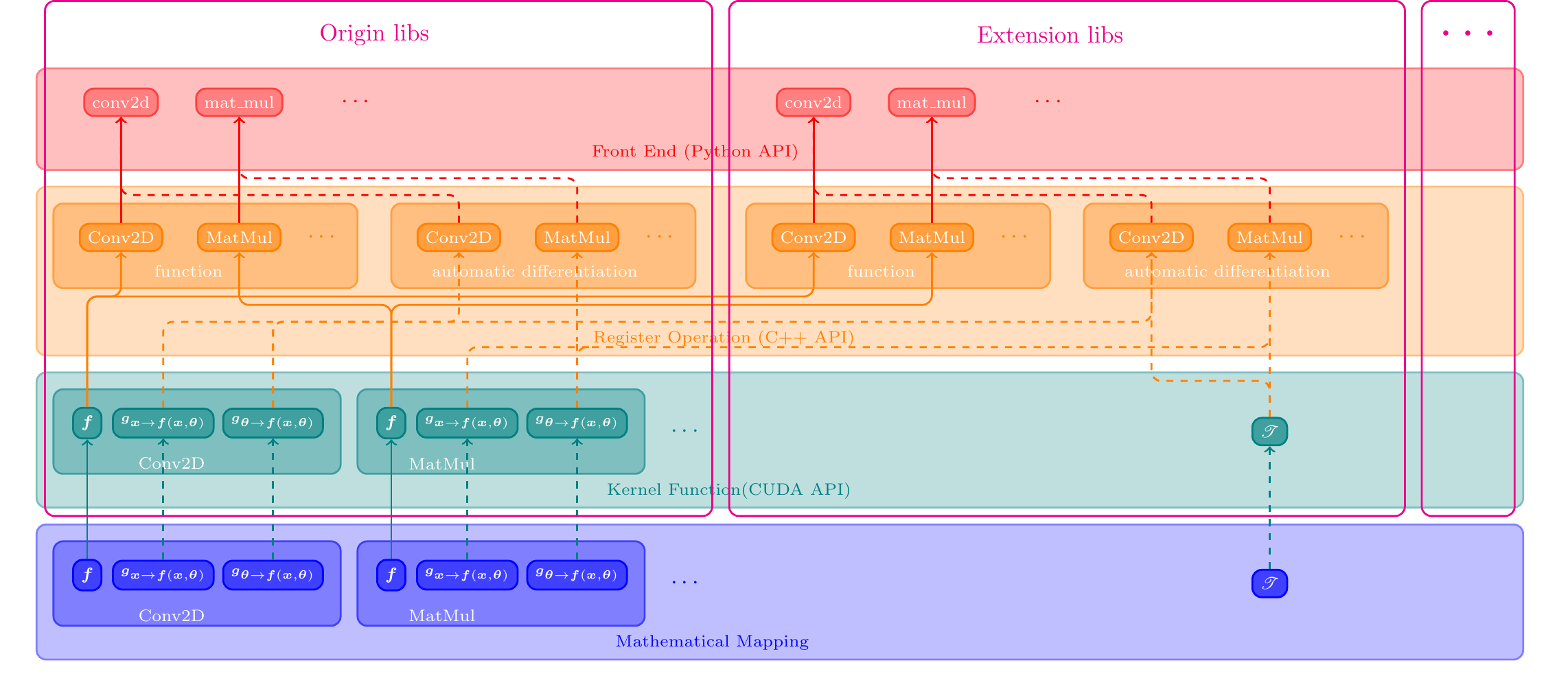}
	\caption{Layered architecture of Tensorflow}
	\label{fig:Encapsulation-1}
\end{figure}
\subsection{Modify the topology of the calculation graph}
According to Eqn.(\ref{virtual-gradient}) and Fig. \ref{fig:Computational-Graphs-Backward}, the computational graph of the virtual gradients can be obtained by adding new nodes on the computational graph of the gradients and reroute the inputs and outputs of new nodes. cf. Fig. \ref{fig:Encapsulation-2}.
\begin{figure}[H]
	\centering
	\includegraphics[width=0.8\textwidth, angle=0]{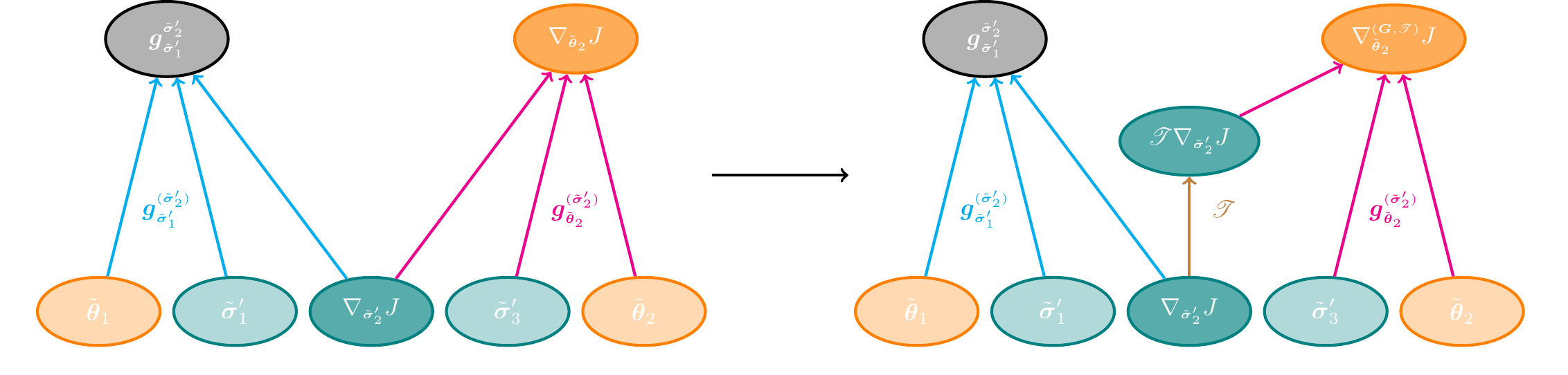}
	\caption{
		Subgraph views of gradients and virtual gradients. \textbf{Left:} the part of the computational graph of the gradients. \textbf{Right:} the part of the computational graph of the virtual gradients.
	}
	\label{fig:Encapsulation-2}
\end{figure}
\section{Convergence Analysis}
\label{sec:Convergence_Analysis}
In this section, we will analyze the theoretical convergence of Eqn.(\ref{first-order-formula-for-theta}) under some assumptions.
\begin{lemma} \label{lemma}
	Let $\boldsymbol{M}$ be a random ($m \times n$)-matrix, $\boldsymbol{M}_{11}, \cdots, \boldsymbol{M}_{ij}, \cdots,\boldsymbol{M}_{mn}$ be an i.i.d. variable from $U(-\infty,+\infty)$. Then
	\begin{equation}
	f_{\lambda} (\boldsymbol{v}, \boldsymbol{u}) := E \left[ \boldsymbol{v}^T \boldsymbol{M}^T \boldsymbol{M} \boldsymbol{u} \big| \ ||\boldsymbol{M}||_2 = \lambda \right] \propto \boldsymbol{v}^T \boldsymbol{u}, \quad \forall \lambda > 0,\forall \boldsymbol{v}, \boldsymbol{u} \in \mathbb{R}^n.
	\end{equation}
\end{lemma}
\begin{proof}
	Let $\boldsymbol{e}_i$ be the unit vector whose i-th component is 1, $f_{\lambda}$ is bilinear, Then
	\begin{equation}
	f_{\lambda} (\boldsymbol{v}, \boldsymbol{u}) 
	= \sum_{i=1}^{n} \sum_{j=1}^{n} \boldsymbol{v}_i \boldsymbol{u}_j f_{\lambda} (\boldsymbol{e}_i, \boldsymbol{e}_j) 
	= \boldsymbol{v}^T \boldsymbol{C} \boldsymbol{u}, 
	\quad \boldsymbol{C}_{ij} := E \left[ \boldsymbol{e}_i^T \boldsymbol{M}^T \boldsymbol{M} \boldsymbol{e}_j \big| \ ||\boldsymbol{M}||_2 = \lambda \right].
	\label{lemma-proof-1}
	\end{equation}
	Since $\boldsymbol{M}_{11}, \cdots, \boldsymbol{M}_{ij}, \cdots,\boldsymbol{M}_{mn}$ be an i.i.d. variable from $U(-\infty,+\infty)$, the following holds true:
	\begin{equation*}
	\left\{
	\begin{array}{l}
	E [ \boldsymbol{e}_1^T \boldsymbol{M}^T \boldsymbol{M} \boldsymbol{e}_1 | \ ||\boldsymbol{M}||_2 = \lambda ] 
	= \cdots
	= E [ \boldsymbol{e}_n^T \boldsymbol{M}^T \boldsymbol{M} \boldsymbol{e}_n | \ ||\boldsymbol{M}||_2 = \lambda ]
	:= c_0 > 0  \\ 
	E [ \boldsymbol{e}_i^T \boldsymbol{M}^T \boldsymbol{M} \boldsymbol{e}_j | \ ||\boldsymbol{M}||_2 = \lambda ] 
	= 0, \quad  i, j \in \{1, \cdots, n\} \\ 
	\end{array}
	\ \right. .
	\end{equation*}
	Thus:
	\begin{equation}
	f_{\lambda} (\boldsymbol{v}, \boldsymbol{u}) = c_0 \boldsymbol{v}^T \boldsymbol{u}.
	\label{lemma-proof-2}
	\end{equation}
\end{proof}
Fig. \ref{fig:Lemma} proof our lemma.
\begin{figure}[H]
	\centering
	\includegraphics[width=0.5\textwidth, angle=0]{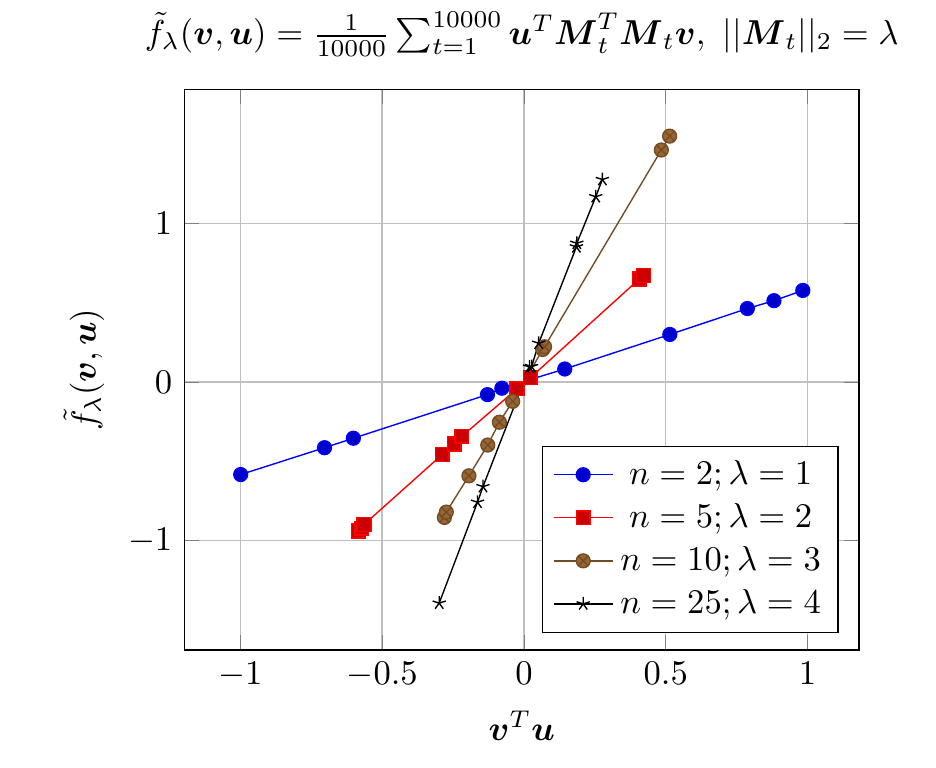}
	\caption{
		The relationship between $\boldsymbol{v}^T \boldsymbol{u}$ and the estimate of $f_{\lambda}(\boldsymbol{v}, \boldsymbol{u})$. Each point corresponds to a pair of random vector $(\boldsymbol{v}, \boldsymbol{u})$ and a random matrix set $\{\boldsymbol{M}_t | t=1, \cdots, 10000 \}$.
	}
	\label{fig:Lemma}
\end{figure}
\begin{corollary} \label{corollary}
	For $\boldsymbol{M}$ defined in Lemma \ref{lemma}, if $\boldsymbol{v}^T \boldsymbol{u} > 0$, then:
	\begin{equation}
	E \left[ \boldsymbol{v}^T \boldsymbol{M}^T \boldsymbol{M} \boldsymbol{u} \right] > 0.
	\end{equation}
\end{corollary}
\begin{theorem} \label{theorem}
	Let $F$ and $\boldsymbol{f}$ be second-order differentiable functions with random variables in their expression, we set:
	\begin{equation*}
	\boldsymbol{v}^T \mathscr{T} \boldsymbol{v} > 0,\quad 
	i \in {1, \cdots, m},\ \forall \boldsymbol{v} \in \mathbb{R}^{m} \backslash \{ \boldsymbol{0} \}.
	\end{equation*}
	If each component of Jacobian matrix $\boldsymbol{M} = \left( \frac{\partial \boldsymbol{f}_i(\boldsymbol{\theta})}{\partial \boldsymbol{\theta}_j} \right)_{m \times n}$ is an i.i.d. variable from $U(-\infty,+\infty)$,
	then, for $\boldsymbol{\theta}^{(t+1)} 
	= \boldsymbol{\theta}^{(t)} - \alpha * \boldsymbol{M} \mathscr{T} \nabla_{\boldsymbol{f}(\boldsymbol{\theta})} F(\boldsymbol{f}(\boldsymbol{\theta}))|_{\boldsymbol{\theta} = \boldsymbol{\theta}^{(t)}}$ and $\nabla_{\boldsymbol{\theta}} F(\boldsymbol{f}(\boldsymbol{\theta})|_{\boldsymbol{\theta} = \boldsymbol{\theta}^{(t)}} \ne \boldsymbol{0}$  there exists a $\alpha > 0$  such that
	\begin{equation*}
	E \left[ F(\boldsymbol{f}(\boldsymbol{\theta}^{(t+1)})) - F(\boldsymbol{f}(\boldsymbol{\theta}^{(t)})) \right] < 0,
	\end{equation*}
\end{theorem}
\begin{proof}
	Without loss of generality, we can assume $\boldsymbol{\theta} \in \mathbb{R}^n, \boldsymbol{\sigma} = \boldsymbol{f}(\boldsymbol{\theta}) \in \mathbb{R}^m, m > n$. Then, the Maclaurin series for $F(\boldsymbol{f}(\boldsymbol{\theta})))$ around the point $\boldsymbol{\theta}^{(t)}$ is:
		\begin{align*}
	F(\boldsymbol{f}(\boldsymbol{\theta}^{(t+1)})) - F(\boldsymbol{f}(\boldsymbol{\theta}^{(t)})) 
	&= -\alpha (\nabla_{\boldsymbol{\theta}} J|_{\boldsymbol{\theta} = \boldsymbol{\theta}^{(t)}})^T
	\boldsymbol{M} \mathscr{T} \nabla_{\boldsymbol{\sigma}} J|_{\boldsymbol{\theta} = \boldsymbol{\theta}^{(t)}} + o(\alpha) \\
	&= -\alpha (\nabla_{\boldsymbol{\sigma}} J|_{\boldsymbol{\theta} = \boldsymbol{\theta}^{(t)}})^T
	\boldsymbol{M}^T \boldsymbol{M} \mathscr{T} \nabla_{\boldsymbol{\sigma}} J|_{\boldsymbol{\theta} = \boldsymbol{\theta}^{(t)}} + o(\alpha).
	\end{align*}
	Let $\boldsymbol{v} = \nabla_{\boldsymbol{\sigma}} J |_{\boldsymbol{\theta} = \boldsymbol{\theta}^{(t)}}$. According to corollary \ref{corollary}:
	\begin{equation*}
	E \left[ F(\boldsymbol{f}(\boldsymbol{\theta}^{(t+1)})) - F(\boldsymbol{f}(\boldsymbol{\theta}^{(t)})) \right] = - \alpha E \left[ \boldsymbol{v}^T \boldsymbol{M}^T \boldsymbol{M} \mathscr{T} \boldsymbol{v} \right] + o(\alpha) < 0.
	\end{equation*}
\end{proof}
Although our convergence analysis in Thm.\ref{theorem} only applies to the assumption of uniform distribution, we empirically found that SVGD often outperforms other methods in general cases. 
\section{Related Work}
\label{sec:Related_Work}
\textbf{First-order methods.} For general first-order methods, The moving direction  $\boldsymbol{p}^{(t)}$ of the variables can be regarded as the function of the stochastic gradient $\boldsymbol{g}^{(t)}$:
\begin{itemize}
	\item \textbf{SGD:}~~~~$\boldsymbol{p}^{(t)} = \mathscr{T}\boldsymbol{g}^{(t)} := -\boldsymbol{g}^{(t)}$.
	\item \textbf{Momentum:}\cite{polyak1964some}~~~~Let $\boldsymbol{m}^{(0)} = \boldsymbol{0}, \boldsymbol{m}^{(t+1)} = c\ \boldsymbol{m}^{(t)} + \boldsymbol{g}^{(t)}$. Then:
	\begin{equation*}
		\boldsymbol{p}^{(t)} = \mathscr{T}\boldsymbol{g}^{(t)} := -\boldsymbol{m}^{(t+1)}.
	\end{equation*}
	\item \textbf{RMSProp:}~~~~Let $\boldsymbol{r}^{(0)} = \boldsymbol{0}, \boldsymbol{r}^{(t+1)} = \rho\ \boldsymbol{r}^{(t)} + (1 - \rho)\ \boldsymbol{g}^{(t)} \odot \boldsymbol{g}^{(t)}$. Then:
	\begin{equation*}
		\boldsymbol{p}^{(t)} = \mathscr{T}\boldsymbol{g}^{(t)} := - \frac{\boldsymbol{g}^{(t)}}{\sqrt{\delta + \boldsymbol{r}^{(t+1)}}}.
	\end{equation*}
	\item \textbf{Adam:}~~~~Let $\boldsymbol{s}^{(0)} = \boldsymbol{r}^{(0)} = \boldsymbol{0}, \boldsymbol{s}^{(t+1)} = \rho_1\ \boldsymbol{s}^{(t)} + (1 - \rho_1)\ \boldsymbol{g}^{(t)}, \boldsymbol{r}^{(t+1)} = \rho_2\ \boldsymbol{r}^{(t)} + (1 - \rho_2)\ \boldsymbol{g}^{(t)} \odot \boldsymbol{g}^{(t)}$. Then:
	\begin{equation*}
		\boldsymbol{p}^{(t)} = \mathscr{T}\boldsymbol{g}^{(t)} := -\frac{\boldsymbol{s}^{(t+1)}/(1 - \rho_1^t)}{\sqrt{\delta + \boldsymbol{r}^{(t+1)}/(1 - \rho_2^t)}}.
	\end{equation*}
\end{itemize}
However, in SVGD method, $\boldsymbol{p}^{(t)} = -\nabla_{\boldsymbol{\theta}}^{(\boldsymbol{G}, \mathscr{T})} J$ cannot be written as a function of $\boldsymbol{g}^{(t)}$. Thus, SVGD is not essentially a first-order method.

\textbf{Global minimum.} A central challenge of non-convex optimization is avoiding sub-optimal local minima. Although it has been shown that the variable can sometimes converges to a neighborhood of the global minimum by adding noise\cite{neelakantan2015neural,neelakantan2015adding,kurach2015neural,kaiser2015neural,zeyer2017comprehensive}, the convergence rate is still a problem. Note that the DP method has some probability to escape “appropriately shallow” local minima because the moving direction of the variable is generated by solving several sub-problems instead of the original problem. We use computational graph and automatic differentiation to generate the sub-problems in DP, such as what we did in the SVGD method.
\section{Experiments}
\label{sec:Experiments}
In this section, we evaluated our method on two benchmark datasets using several different neural network architectures. We train the neural networks using RMSProp, Adam, SGD, and SVGD to minimize the cross-entropy objective function with $L_1$ weight decay on the parameters to prevent over-fitting. To be fair, for different methods, a given objective function will be minimized with different learning rates. All extension libs, algorithm, and experimental logs in this paper can be found at the URL: \url{https://github.com/LizhengMathAi/svgd}. 

The following experiments show that SVGD has a relatively faster convergence rate and better test accuracy than SGD, RMSProp, and Adam. 

\subsection{Multi-layer neural network}
In our first set of experiments, we train a 5-layer neural network (Fig. \ref{fig:MLP}) on the MNIST \cite{deng2012mnist} handwritten digit classification task.
\begin{figure}[H]
	\centering
	\includegraphics[width=0.85\textwidth, angle=0]{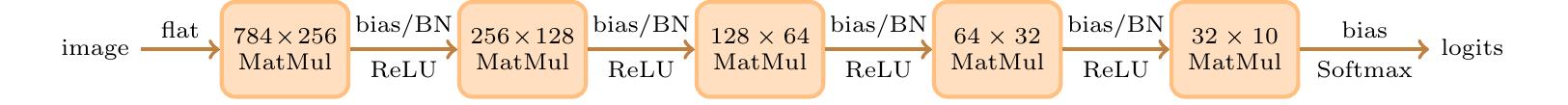}
	\caption{MLP architecture for MNIST with 5 parameter layers (245482 params).}
	\label{fig:MLP}
\end{figure}
The model is trained with a mini-batch size of 32 and weight decay of $1.0 \times 10^{-4}$. In Table \ref{tab:MNIST-MLP}, we decay $\alpha$ at 1.6k and 3.6k iterations and summarize the optimal learning rates for RMSProp, Adam, SGD, and SVGD by hundreds of experiments.
\begin{table}[H]
	\centering
	\begin{tabular}{l|l|p{15mm}|p{15mm}|p{15mm}|p{20mm}}
		\hline
		\multicolumn{2}{l|}{ } & \textbf{RMSProp} & \textbf{Adam} & \textbf{SGD} & \textbf{SVGD}(s=0.1) \\ 
		\hline
		\multirow{3}{*}{$\alpha$} & \textbf{iter:} $[0\sim1599]$ & 0.001 & 0.001 & 0.1 & 0.01 \\
		& \textbf{iter:} $[1600\sim3599]$ & 0.0005 & 0.00005 & 0.05 & 0.005 \\
		& \textbf{iter:} $[1600\sim5999]$ & 0.00005 & 0.00005 & 0.01 & 0.001 \\
		\hline
		\multicolumn{2}{c|}{\textbf{test top-1 error}} & 1.80\% & 1.94\% & 1.76\% & 1.60\% \\
		\hline
	\end{tabular}
	\caption{The test error and learning rates in \textbf{MLP} experiments.}
	\label{tab:MNIST-MLP}
\end{table}
In Table \ref{tab:MNIST-MLP} and Fig. \ref{fig:MNIST-MLP} we compare the error rates and their descent process process on the CIFAR-10 test set, respectively.
\begin{figure}[H]
	\centering
	\includegraphics[width=0.85\textwidth, angle=0]{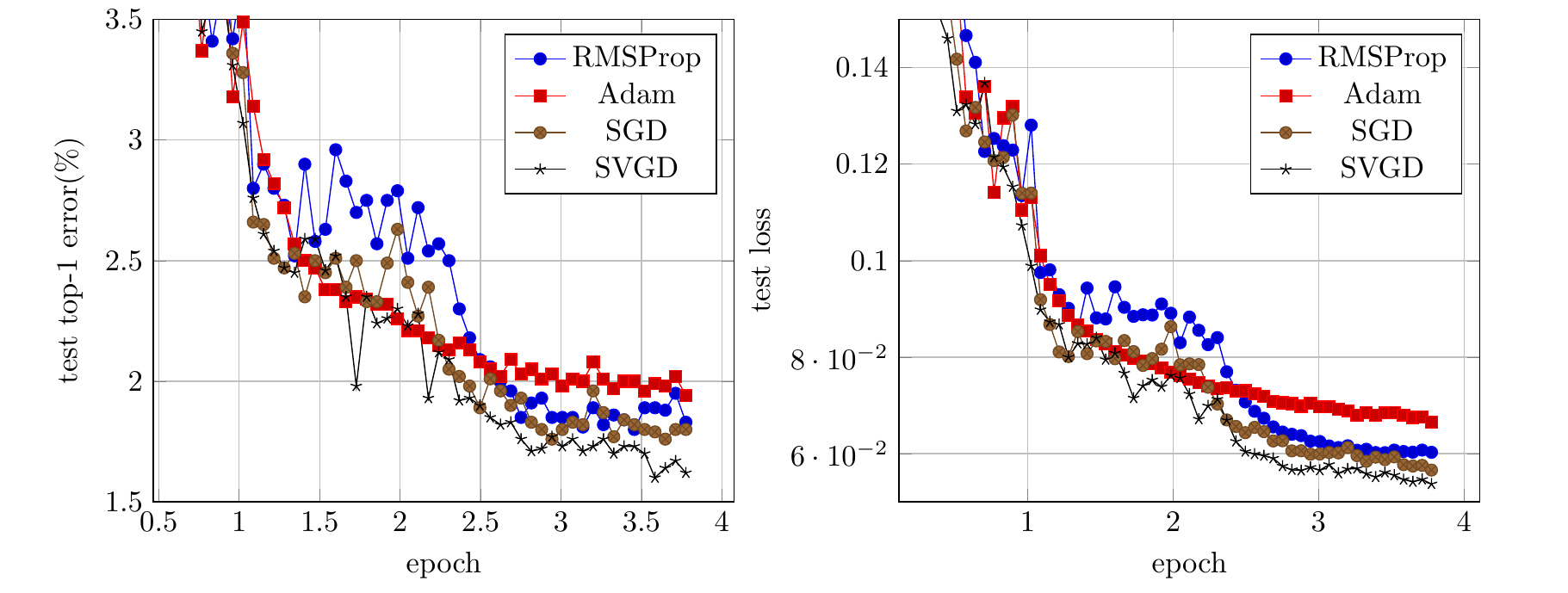}
	\caption{Comparison of first-order methods on \textbf{MNIST} digit classification for 3.75 epochs.}
	\label{fig:MNIST-MLP}
\end{figure}
\subsection{Convolutional neural network}
We train a VGG model (Fig. \ref{fig:VGG}) on the CIFAR-10 \cite{krizhevsky2009learning} classification task and follow the simple data augmentation in \cite{lee2015deeply,he2016deep} for training and evaluate the original image for testing.
\begin{figure}[H]
	\centering
	\includegraphics[width=0.85\textwidth, angle=0]{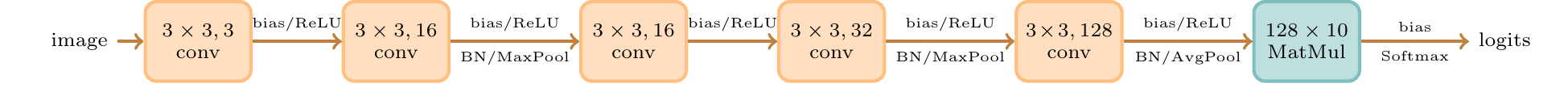}
	\caption{VGG model architecture for \textbf{CIFAR-10} with 6 parameter layers (46126 params).}
	\label{fig:VGG}
\end{figure}
The model is trained with a mini-batch size of 128 and weight decay of $1.0 \times 10^{-5}$. In Table \ref{tab:CIFAR10-VGG}, we decay $\alpha$ at 12k and 24k iterations and summarize the optimal learning rates for RMSProp, Adam, SGD, and SVGD by hundreds of experiments.
\begin{table}[H]
	\centering
	\begin{tabular}{l|l|p{15mm}|p{15mm}|p{15mm}|p{23mm}}
		\hline
		\multicolumn{2}{l|}{ } & \textbf{RMSProp} & \textbf{Adam} & \textbf{SGD} & \textbf{SVGD}(s=0.001) \\ 
		\hline
		\multirow{3}{*}{$\alpha$} & \textbf{iter:} $[0\sim11999]$ & 0.02 & 0.02 & 2.0 & 2.0 \\
		& \textbf{iter:} $[12000\sim23999]$ & 0.01 & 0.01 & 0.5 & 0.5 \\
		& \textbf{iter:} $[24000\sim34999]$ & 0.002 & 0.005 & 0.005 & 0.005 \\
		\hline
		\multicolumn{2}{c|}{\textbf{test top-1 error}} & 17.78\% & 18.02\% & 17.32\% & 17.07\% \\
		\hline
	\end{tabular}
	\caption{The test error and learning rates in \textbf{VGG} experiment.}
	\label{tab:CIFAR10-VGG}
\end{table}
In Table \ref{tab:CIFAR10-VGG} and Fig. \ref{fig:CIFAR10-VGG} we compare the error rates and their descent process on the CIFAR-10 test set, respectively.
\begin{figure}[H]
	\centering
	\includegraphics[width=0.85\textwidth, angle=0]{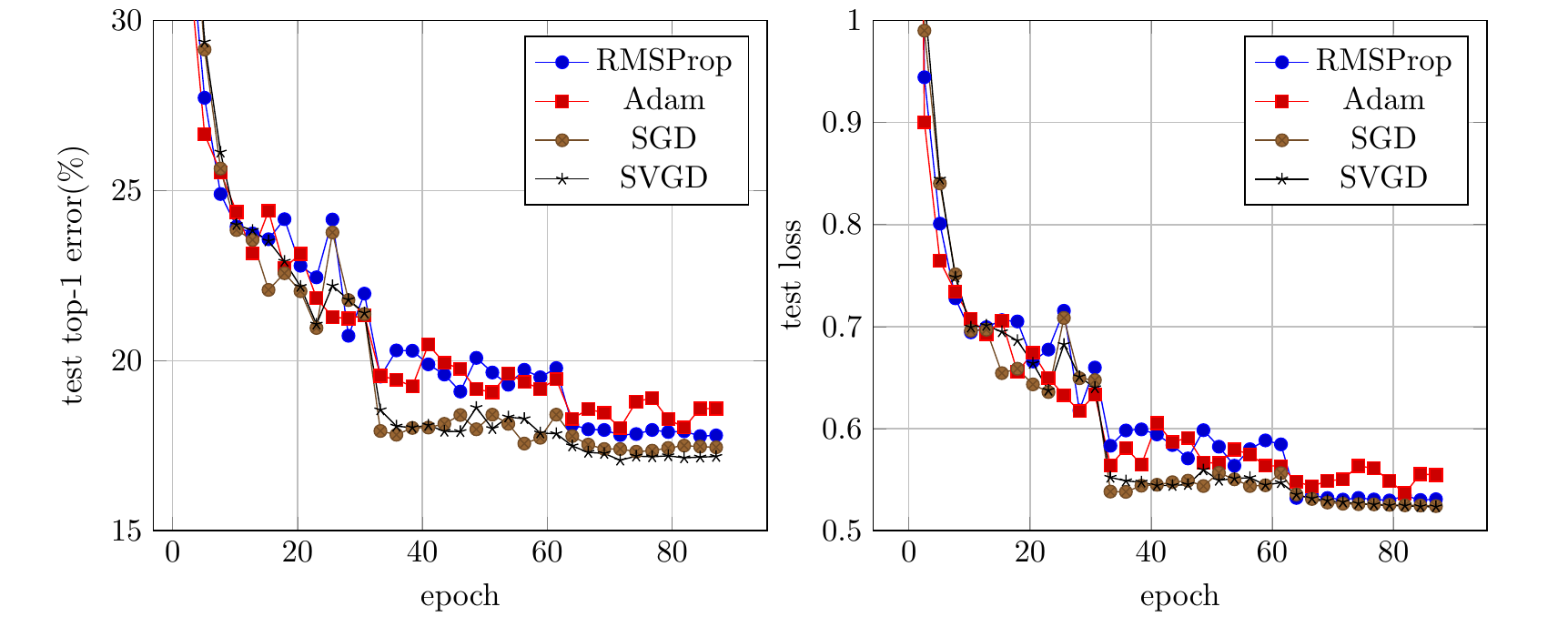}
	\caption{Comparison of first-order methods on \textbf{CIFAR-10} dataset for 90 epochs.}
	\label{fig:CIFAR10-VGG}
\end{figure}
\subsection{Deep neural network}
We use the same hyperparameters with \cite{he2016deep} to train ResNet-20 model(0.27M params) on the CIFAR-10 classification task. In Table \ref{tab:CIFAR10-ResNet}, we decay $\alpha$ at 12k and 24k iterations and summarize the optimal learning rates for RMSProp, Adam, SGD, and SVGD by hundreds of experiments.
\begin{table}[H]
	\centering
	\begin{tabular}{l|l|p{15mm}|p{15mm}|p{15mm}|p{22mm}}
		\hline
		\multicolumn{2}{l|}{ } & \textbf{RMSProp} & \textbf{Adam} & \textbf{SGD} & \textbf{SVGD}(s=0.01) \\ 
		\hline
		\multirow{3}{*}{$\alpha$} & \textbf{iter:} $[0\sim31999]$ & 0.001 & 0.001 & 0.1 & 0.5 \\
		& \textbf{iter:} $[32000\sim41999]$ & 0.0001 & 0.0001 & 0.01 & 0.02 \\
		& \textbf{iter:} $[42000\sim49999]$ & 0.0001 & 0.00005 & 0.001 & 0.01 \\
		\hline
		\multicolumn{2}{c|}{\textbf{test top-1 error}} & 11.18\% & 11.12\% & 10.69\% & 8.62\% \\
		\hline
	\end{tabular}
	\caption{The test error and learning rates in \textbf{ResNet} experiments.}
	\label{tab:CIFAR10-ResNet}
\end{table}
In Table \ref{tab:CIFAR10-ResNet} and Fig. \ref{fig:CIFAR10-ResNet} we compare the error rates and their descent process on the CIFAR-10 test set, respectively. The top-1 error fluctuations in experiments do not exceed 1\%. See \cite{srivastava2015training} for more information on the CIFAR-10 experimental record.
\begin{figure}[H]
	\centering
	\includegraphics[width=0.85\textwidth, angle=0]{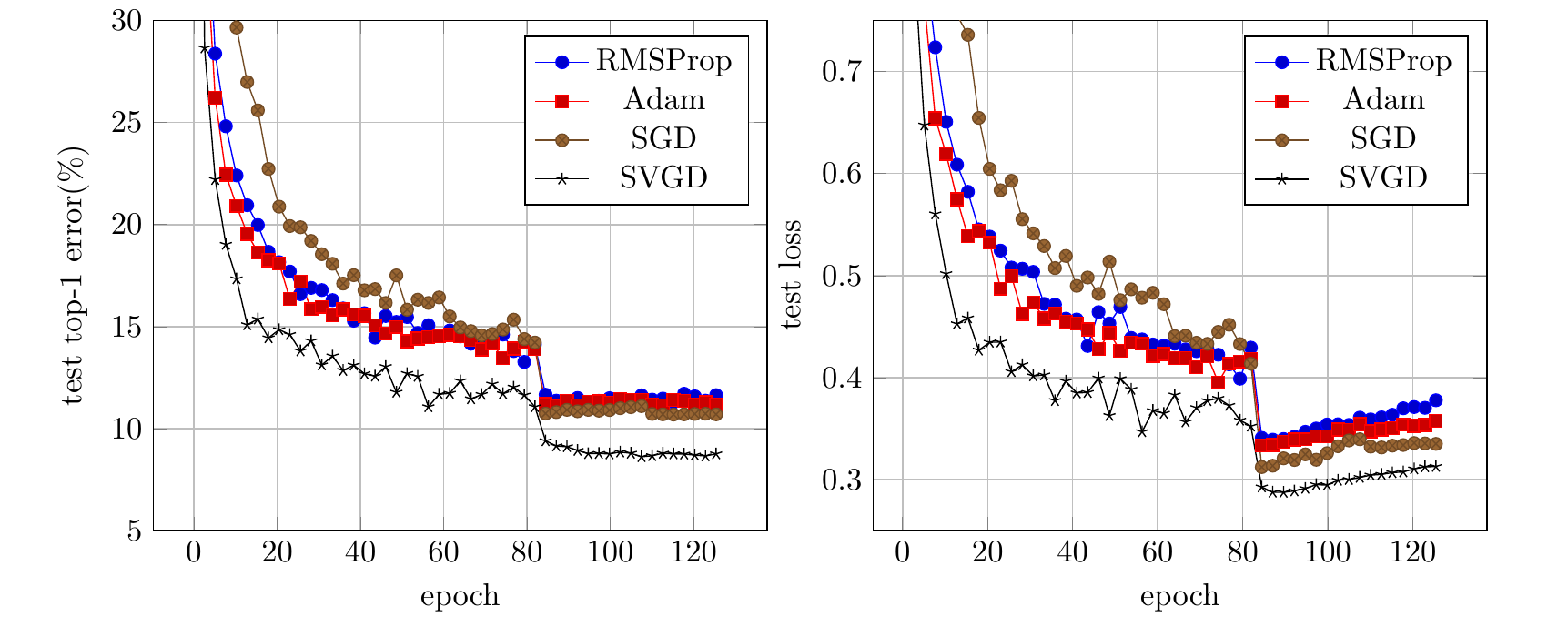}
	\caption{Comparison of first-order methods on \textbf{CIFAR-10} dataset for 125 epochs.}
	\label{fig:CIFAR10-ResNet}
\end{figure}

\bibliographystyle{unsrt}  
\bibliography{references} 


\end{document}